\documentclass[pmlr]{jmlr}

\usepackage{booktabs}
\usepackage[load-configurations=version-1]{siunitx} 

\theorembodyfont{\upshape}
\theoremheaderfont{\scshape}
\theorempostheader{:}
\theoremsep{\newline}
\newtheorem*{note}{Note}

\jmlrvolume{182}
\jmlryear{2022}
\jmlrsubmitted{submission date}
\jmlrpublished{publication date}
\jmlrworkshop{Machine Learning for Healthcare} 

\usepackage[utf8]{inputenc}

\usepackage{algorithm,algorithmicx,algpseudocode} 
\usepackage{amsmath,amssymb}

\providecommand{\norm}[1]{\bigl\lVert#1\bigl\rVert}

\title[Ensembling Neural Networks for Improved Prediction and Privacy]{Ensembling Neural Networks for Improved Prediction and Privacy in Early Diagnosis of Sepsis}

 \author{\Name{Shigehiko Schamoni} \Email{schamoni@cl.uni-heidelberg.de}\\
  \Name{Michael Hagmann} \Email{hagmann@cl.uni-heidelberg.de}\\
  \Name{Stefan Riezler} \Email{riezler@cl.uni-heidelberg.de}\\
  \addr Department of Computational Linguistics and\\ Interdisciplinary Center for Scientific Computing (IWR)\\ Heidelberg University, Germany}

\date{July 2022}

\makeatletter
 \let\Ginclude@graphics\@org@Ginclude@graphics 
\makeatother

\begin{document}

\maketitle

\begin{abstract}
 Ensembling neural networks is a long-standing technique for improving the generalization error of neural networks by combining networks with orthogonal properties via a committee decision. We show that this technique is an ideal fit for machine learning on medical data: First, ensembles are amenable to parallel and asynchronous learning, thus enabling efficient training of patient-specific component neural networks. Second, building on the idea of minimizing generalization error by selecting uncorrelated patient-specific networks, we show that one can build an ensemble of a few selected patient-specific models that outperforms a single model trained on much larger pooled datasets. Third, the non-iterative ensemble combination step is an optimal low-dimensional entry point to apply output perturbation to guarantee the privacy of the patient-specific networks. We exemplify our framework of differentially private ensembles on the task of early prediction of sepsis, using real-life intensive care unit data labeled by clinical experts.
\end{abstract}

\section{Introduction}

Ensembling describes a family of algorithms that train multiple learners to solve the same problem, and exploit their heterogeneous properties to perform a committee-based prediction that achieves higher accuracy than any single component learner. These techniques are well-tried in machine learning practice and have led to theoretically well-founded algorithms such as stacking \citep{Wolpert:92}, boosting \citep{FreundSchapire:95}, or bagging \citep{Breiman:96}. Research on ensembling has very early tackled the problem of reducing variance of neural networks while keeping bias low at the same time. In the wide spectrum of approaches, ranging from sophisticated techniques to jointly train component networks \citep{LiuYao:99,BuschjaegerETAL:20} to building ensembles from model parameters of a single training trajectory \citep{HuangETAL:17,IzmailovETAL:18}, we are specifically interested in approaches where component models are trained independently and then smartly combined.  

A key insight in this area, first formulated in \cite{PerroneCooper:92}, is that the generalization error of the weighted average of predictions of individual component networks can be formalized as the weighted correlation between the component neural networks participating in the ensemble. This formulation opens several possibilities for efficient and effective machine learning: First, the bulk of the machine learning cost, namely the cost of training individual component networks, can be trivially parallelized or even be done asynchronously, thus providing an efficient way of enhancing the representational power of the ensemble by training multiple classifiers at once. Second, optimizing combination weights to minimize the weighted correlation between component networks provides a direct avenue to minimize the generalization error of the ensemble, or to build a sparse ensemble from the optimal subset of component networks with small error and small correlation with other component networks. 

A further advantage of weighted-averaging ensembles that has been investigated much less than their generalization performance is the possibility to seamlessly integrate privacy protection into machine learning.  
In the case of machine learning models trained on medical data, the privacy to be protected might concern the membership of patient-specific data in the training data for a particular disease. As argued by \cite{DinurNissim:03}, removal of ``identifying'' attributes such as patients' names is not enough, but instead random perturbations have to be applied to the outputs in order to protect privacy even in the simplest case of ``statistical'' queries such as averages over databases. The framework of differential privacy \citep{DworkRoth:14} allows giving strong guarantees on the information derivable from private training data when querying a machine learning algorithm.
We show that weighted-averaging ensembles do possess small sensitivity by tightly bounded output ranges and do not accumulate privacy budget via iterative training, thus they are ideally suited for privacy protection at small noise scales. Furthermore, we prove that uniform weights are optimal to protect privacy in a weighted-averaging ensemble. 

Specifying guarantees on privacy protection is of increasing importance for medical research. 
National laws and regulations such as the US HIPAA Privacy rule\footnote{\url{www.hhs.gov/hipaa/for-professionals/privacy/} (accessed 07/06/2022)} require measures to protect the privacy of health information. 
On the hospital level, protecting a patient's privacy is crucial especially when information is shared across institutions. Our method demonstrates the benefit of \emph{output sharing} where hospitals keep their in-house model in a secured area and only share the output with other institutions, thus avoiding the challenges and difficulties of \emph{model sharing} techniques such as federated learning \citep{RiekeETAL20}. 
On the patient level, a recent survey has shown that more than 30\% of the participants are comfortable with sharing their electronic health data for personalized healthcare, while less than 5\% are very uncomfortable with sharing \citep{GarettYoung22}. This means more than 60\% do not have a strong opinion on this topic, thus we hope that an increasing number of people will share their data if stronger privacy guarantees can be given.

\subsection*{Generalizable Insights about Machine Learning in the Context of Healthcare}

Expert labels and neural networks are a powerful combination for early sepsis prediction. However, patient data for this  task is scarce as expert labels are difficult to obtain, while the protection of privacy is crucial to encourage patients to contribute with their personal private data.
We show how to train individual personalized models and how to combine a small number of patient models in an ensemble that has more desirable properties in the field of medical data analysis than a standard full model, i.e., a single model that is trained on all available patient data.
\begin{itemize}
    \item We present theoretical results that an ensemble of models which was trained on a fraction of the available data can be better than a full model, and we verify this empirically.
    \item Our training method not only exposes fewer patients in the predictor than a full model, but also protects the privacy better: we apply a strong membership attack and show that the ensemble successfully prevents privacy leakage.
    \item We show that an ensemble of several models is favorable to a single model due to its reduced sensitivity in theory, and we experimentally verify that our ensemble maintains its accuracy at privacy budgets almost two orders of magnitude smaller than a full model.
\end{itemize}
Furthermore, our ensemble can be easily updated by model-growing without the need of retraining the whole system when new patient's data becomes available.

\section{Related Work}

Ensembles of neural networks have been researched at least since \cite{HansenSalamon:90}, and are now a standard tool of deep learning. The spectrum of approaches ranges from joint training of component models under ensemble objectives such as negative correlation learning \citep{LiuYao:99,BuschjaegerETAL:20} to approaches to efficiently build ensembles by combining snapshots of model parameters along the training trajectory of a single network by averaging in model space \citep{HuangETAL:17}, or weight space \citep{IzmailovETAL:18}. Even well-known staples such as dropout can be seen as ensembles of subnetworks \citep{SrivastavaETAL:14}. For a recent overview over ensemble deep learning, see \cite{GanaieETAL:21}.

Traditionally, ensemble methods are often used in medical data science. Recent examples can be found in the area of early prediction of sepsis: \cite{GohETAL:21} use a voting ensemble of a logistic regression model and a random forest trained on same data; \cite{MoorETAL:21} use a max-score ensemble of four machine learning models trained on four different datasets. The privacy preserving aspects of ensemble methods in health care, however, have only been investigated recently. \cite{FritchmanETAL18} describe a framework for privacy preserving inference using cryptographic protocols. \cite{AdamsETAL22} demonstrate secure training of ensembles in a multi-party computation scenario. Both works are based on decision tree ensembles, while our ensemble strategy can be combined with any machine learning model.

Differential privacy has become a de-facto standard for privacy protection at least since \cite{Dwork:06}. This framework has been applied to the case of  privacy protection in machine learning where the goal is to protect the privacy of training data when publicly releasing a machine learning model.
To our knowledge, despite their natural fit to protect privacy in the combination of patient-specific models, differentially private ensembles have not yet been widely used in medical data science. Instead, the paradigm of cryptography still seems to be going strong in the area of collaborative learning on health data \citep{GongETAL:15}.

\section{Methods}

In this section, we first describe the basic machine learning model that is used throughout our experiments. We then explain how we combine multiple models in an ensemble that improves prediction accuracy, and we also show how ensembles can enhance differential privacy. Based on the theoretical results, we finally design an algorithm for greedy ensemble growing. 

\subsection{Recurrent Neural Networks over Time Series}

Our basic machine learning models are recurrent neural network (RNN) architectures that predict a severity score $y_t$ for each time step $t$ (see Figure \ref{fig:rnnarchitecture}). We use discretized 30 minute steps as input for our models and predict a sepsis-related score. The predicted score at each time step is an expert label ranging from 0 to 4. 
The motivation for using an RNN-based system is the intuition that a recurrent network is able to model a dynamic system over time \citep{DurstewitzETAL:18} (i.e., the development of the patient) with its feedback loop connections. 
We implement a special form of RNNs, namely LSTMs \citep{HochreiterSchmidhuber:97}, due to their ability to model both short and long term dependencies in time series. However, our ensemble growing strategy can easily be adapted to other sequence-to-sequence models such as Transformers \citep{VaswaniETAL17}, or even to time-agnostic models such as feed-forward nets or decision trees. 
Details on the architecture and meta-parameters of our model are given in Appendix \ref{app1}.

\begin{figure}[t]
	\centering
	\includegraphics[width=4cm]{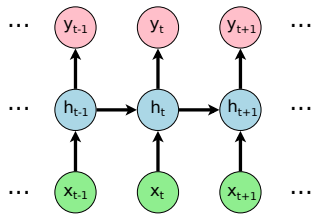}
    \caption{Schematic view on the RNN architecture. For each input $\mathbf{x}$, the RNN is in a hidden state $\mathbf{h}$ that is also conditioned on the previous hidden state and generates an output $y$. 
    }
    \label{fig:rnnarchitecture}
\end{figure}

We train two different types of models with a regression objective. First, a  \emph{full model} trained on all available data, and secondly, an \emph{ensemble} of individual models that are each trained on the data of a single patient. 
The training data for the full model consist of 638 or 637 patient's timelines depending on the data fold (see Section \ref{sec:datacohort}). It is trained for a maximum of 200 epochs with an early stopping criterion on the validation set. We originally optimized the general model architecture on this model type and use the same architecture throughout our experiments. 

\subsection{Weighted Averaging Ensemble}
\label{sec:weightedaverageensemble}

In our setup, we are combining models that were trained on individual patients, which makes each model an ``expert'' \citep{JacobsETAL91} of sepsis prediction for a specific patient. We then combine selected prototypical patient models by weighting their predictions, motivated by the insight that a physician uses past experience to make future decisions. Such a weighted-average ensemble has very useful properties. Specifically, the ensemble's mean squared error (MSE) can be decomposed as a linear combination of the MSE and the covariances of the errors of the individual components models. Furthermore, a closed-form solution for optimal weights is completely determined by the variance-covariance matrix of the model's errors on a held-out set. 
Similar results have been shown by \cite{PerroneCooper:92} for the case of cross-validation, a setup that is usually applied to prevent overfitting of hyperparameters. We use cross-validation in our experiments particularly due to the fact that our medical data is small and thus our learned models easily overfit and hence the test-set performance shows high variance. \cite{ZhouETAL:02} also exploit this decomposition of the MSE to improve generalization when selecting models and determining weights by using a genetic algorithm as a heuristic. 

\begin{definition}
Let $f,\hat{f}_i: \mathbb{R}^m\rightarrow \mathbb{R}, i=1, \ldots, N$. Further let $\mathfrak{F}=\{\hat{f}_i(x) := \widehat{\mathbb{E}[y|x]}_i\}$ be a set of regression estimates of a regression function $f(x) := \mathbb{E}[y|x]$ and $w_i\in \mathbb{R}$ subject to $\sum_{i=1}^N w_i = 1$. Then
\begin{align*}
    \hat{f}_{em}(x) := \sum_{i=1}^N w_i \hat{f}_i(x)  
\end{align*}
is called the \emph{weighted-average ensemble estimator}.
\end{definition}

When comparing two models' predictions to decide which is best, it is useful to define a measure of the length of the misfit vector $\mathbf{m}$. Each element in this vector is given by the error of the model's prediction $f_i$ for input $x$, i.e., the misfit 
$m_i(x) := f_i(x)-f(x)$ of function $f_i(x)$ with respect to the true value $f(x)$. A norm defined on the vector data space then returns the length of the misfit vector, which is in its simplest form the inner product $M = \mathbf{m}^{\intercal}\mathbf{m}$ or the $\ell_2$-norm of $\mathbf{m}$. Conveniently, the length of the misfit vector is equal to the squared error of $f_i$ in this case. A comparison of the $\ell_2$-norm of the misfits of two models is thus a comparison of the squared error, and the resulted ranking of models is equivalent to a ranking w.r.t. their MSE. 

To measure how different two models are, it is again useful to compare two models' predictions $f_i(x)$ and $f_j(x)$ by looking at the misfits of both functions and calculate their covariance $\sigma^2_{ij} = \frac{1}{N} \mathbf{m}_i^{\intercal}\mathbf{m}_j$. The covariance of functions that generate more similar predictions is above 0, of functions that generate more opposite predictions below 0, and of functions that generate orthogonal predictions it is  0. We are specifically interested in the last case, as a system of orthogonal functions improves generalization if combined with a suitable ensemble growing strategy \citep{PerroneCooper:92,ZhouETAL:02}. 

Summarizing, the covariance of the errors of a set of models can be represented by a quadratic matrix $\mathbf{C}$ where each element is given by the covariance of the prediction errors of two individual models or by the mean squared error of a single model on the diagonal:
\begin{align*}
    \text{C}_{ij} := \mathbb{E}_X[(\hat{f}_i(X)-f(X))(\hat{f}_j(X)-f(X))] = \int (\hat{f}_i(X)-f(X))(\hat{f}_j(X)-f(X)) d\text{P}_X
\end{align*}
is symmetric, $C_{ij} = C_{ji}$, and
\begin{align*}
    \text{MSE}(\hat{f}_i) := \mathbb{E}_X[(\hat{f}_i(X)-f(X))^2] = \textrm{C}_{ii} 
\end{align*}
The following theorem shows that the MSE of the ensemble estimator can be expressed in terms of the $C_{ij}$.

\begin{lemma}
Let $\hat{f}_{em}(x)$ be the ensemble estimator constructed from $\mathfrak{F}=\{\hat{f}_i : i=1,\ldots,N\}$ and $\mathbf{C}=[C_{ij}]_{i,j = 1,\ldots,N}$ be the covariance matrix of $\mathfrak{F}$. Then 
\begin{align*}
    \text{MSE}(\hat{f}_{em}(x)) & = \mathbf{w}^\intercal \mathbf{C} \mathbf{w} \\\notag
    & = \sum_{i=1}^N \sum_{j=1}^N w_i w_j C_{ij} \\\notag
    & = \sum_{i=1}^N w_i^2 C_{ii} + 2 \sum_{i=2}^{N} \sum_{j=1}^{i-1} w_i w_j C_{ij} \\\notag
    & = \sum_{i=1}^N w_i^2  \text{MSE}(\hat{f}_i) + 2 \sum_{i=2}^{N} \sum_{j=1}^{i-1} w_i w_j C_{ij}  
\end{align*}
for all $\mathbf{w} = [w_1, \ldots, w_N]^\intercal \in \mathbb{R}^N$.
\end{lemma}
This fact also facilitates an optimality result with respect to the weights that is given in Appendix \ref{app2}. An important consequence of this result is the fact that if the weights are chosen in an optimal way, then $\text{MSE}(\hat{f}_{em})  \leq \min_{j=1,\ldots,N} \text{MSE}(\hat{f}_{j})$.

\subsection{Differentially Private Ensembling}
\label{sec:dpensembling}

Differential privacy for machine learning is commonly conceptualized as protecting the privacy of training data by randomized optimization algorithms that output learned weights. For example, during stochastic gradient descent (SGD) training of machine learning models, privacy-preserving noise can be added to the trained weights, to the learning objective, or the gradient-based weight updates \citep{JayaramanEvans:19}. While convex optimization algorithms like logistic regression allow for output perturbation or objective perturbation \citep{ChaudhuriETAL:11}, non-convex optimization for deep neural networks requires iterative gradient perturbation in addition to gradient clipping \citep{ShokriShmatikov:15,AbadiETAL:16,McMahanETAL:16}. 
Although these strategies provide privacy guarantees for model sharing in theory, there exist several disadvantages in practice:
\begin{enumerate}
     \item The required privacy budget $\epsilon$ for deep learning can be proportional to the size of the target deep learning model, leading to $\epsilon$ values in the order of hundreds of thousands or more. \cite{JayaramanEvans:19} show that the large privacy budgets that are required in practice in non-convex optimization severely undermine the value of privacy guarantees provided by differential privacy. 
    \item Bounding the gradient norm in order to restrict the sensitivity of SGD training for deep neural networks is limited by a tradeoff between privacy protection and network performance, where training speed and prediction accuracy are lowered if the bound on the gradient norm is too tight.
    \item Iterative learning procedures such as SGD do not scale to large numbers of training iterations due to the fundamental composition theorem of \cite{DworkRoth:14} that causes the required privacy budget to accumulate across iterations.
\end{enumerate}

In this work, we take an approach where models are shared by ensembling, and only the final predictions are made public. This allows circumventing most of the above mentioned problems.
\begin{enumerate}
    \item Ensembling allows privacy protection by output perturbation via a Laplace mechanism. This mechanism is independent of the size of the component networks.
    \item In most application cases, outputs of component networks that contribute to an ensemble are either naturally bounded in a certain range or can be thresholded without loss of generality. 
    \item Ensembling is a one-step process that is applied to the outputs of component networks. It is thus not affected by an accumulation of required privacy budgets across training iterations.
\end{enumerate}

The basic mathematical details of differential privacy are described in Appendix \ref{app3}. 
For our work, we especially build upon work on differentially private mean estimators \citep{EpastoETAL:20,BunSteinke:19}
and formalize a privacy-protected weighted averaging algorithm for ensembling as follows. 
Let $\hat{f}_i: \mathcal{D} \rightarrow \mathbb{R}, i=1, \ldots, N$ be functions approximated by $N$ component neural networks. Furthermore, assume bounded outputs $f_i \in [0,B], i=1, \ldots, N$. Lastly, assume an ensembling technique that combines components by weighted averaging, with weights $w_i \in \mathbb{R}_{\geq 0}, i=1, \ldots, N$ and $\sum_{i=1}^N w_i = 1$. Then Algorithm \ref{alg:privateWAE} protects the privacy of an ensemble of $N$ neural networks simply by applying the Laplace mechanism for output perturbation in the averaging phase:

\begin{algorithm2e}
\caption{Private Weighted Averaging Ensemble}\label{alg:privateWAE}

\DontPrintSemicolon
\KwIn{Outputs of component networks $\hat{f}_1, \ldots, \hat{f}_N$, combination weights $w_1, \ldots, w_N$}
    $\hat{f} \leftarrow w_1 \hat{f}_1 + \cdots + w_N \hat{f}_N$ \;
    $\tilde{f} \leftarrow \hat{f} + \mathrm{Lap}\left( \frac{B \cdot \max_{i=1,\ldots,N} w_i}{\epsilon} \right)$ \;
    \KwRet{ $\tilde{f}$ }\;
\end{algorithm2e}

\begin{lemma} For every set of component networks $f_i$ and weights $w_i$, $i=1,\ldots,N$, Algorithm \ref{alg:privateWAE} is $(\epsilon,0)$-differentially private.
\end{lemma}

\begin{proof}
The $\ell_1$-sensitivity $\Delta \hat{f}$ of the weighted averaging function $\hat{f}$ is $B \cdot \max_{i=1,\ldots,N} w_i$. This allows applying a Laplace mechanism to construct a randomized algorithm $\mathcal{A} = \tilde{f}$ with noise drawn from $\mathrm{Lap}(\Delta \hat{f}/\epsilon)$. By Theorem 3.6 of \cite{DworkRoth:14}, we know that the Laplace mechanism preserves $(\epsilon,0)$-differential privacy. 
\end{proof}

As can be seen from the use of a Laplace mechanism, sensitivity of the ensemble output is minimized by choosing uniform weights $w_i = 1/N$. Uniform weights effectively reduce the $\ell_1$-sensitivity $\Delta \hat{f}$ of an ensemble $\hat{f}$ by a factor of $1/N$ compared to single models ($w_i = 1$), including the full model that is trained on all available training data.
Furthermore, these weights guarantee privacy protection at perturbation with minimal variance. This theoretical advantage of ensembles in privacy protection is confirmed in the experiments presented in Section \ref{sec:experiments}.

\subsection{Algorithm}
\label{sec:algorithm}

The method implemented in our work can be characterized as a ``bucket of models'' where we select a pool of individual learners that performed best on a validation set. The idea is to identify prototypical models that represent certain types of patients whose properties can be transferred well to other patients.
Our bucket of models consists of a number of individual models that were selected based on the criterion defined by \cite{PerroneCooper:92}. 
By looking at the \textit{misfit} of function $\hat{f}_i$, which is the deviation from the true solution, $m_i := f(x)-\hat{f}_i(x)$, 
the algorithm adds a new candidate model $f_{\text{new}}$ to the ensemble $\hat{f}_{em}$ if the candidate satisfies the following inequality: 
\begin{equation}
\label{eq:growingensemble}
(2N+1)\text{MSE}[\hat{f}_{em}] > 2 \sum_{i \neq \text{new}} \mathbb{E}[m_{\text{new}} m_i] + \mathbb{E}[m_{\text{new}}^2]
\end{equation}
The left part of the RHS's sum expresses that the candidate has to be reasonably different to already included models while the right part of the sum makes sure that the candidate has low error on the validation set, hence the inclusion of $\hat{f}_{new}$ improves generalization and reduces error of the ensemble. 
The total number of models in the ensemble is not fixed and depends on the performance of the trained models, their diversity, and the validation set.

\begin{algorithm2e}
\caption{Greedy Ensemble Growing}
\label{alg:greedyensemble}
\DontPrintSemicolon
\KwIn{List of patient models $P$ sorted non-decreasing by MSE}
\KwOut{Final ensemble $\hat{f}_{em}$}

  \SetKwFunction{FnextModel}{nextModel}

$\hat{f}_{em}\leftarrow \text{initialize ensemble}$\;
\Repeat{$\hat{f}_{em}$ stops growing}{
  $f_{new} \leftarrow \FnextModel( \hat{f}_{em}, P )$\;
  \If{$f_{new}$ is found}{
    $\hat{f}_{em}\leftarrow \hat{f}_{em} + f_{new}$\;
  }
}
\vspace{1em}
  \SetKwProg{Pn}{Function}{:}{\KwRet {not found}}
  \Pn{\FnextModel{$\hat{f}_{em}$, $P$}}{
    \For{$f_{new}$ in $P$}{
      \If {$(2N+1)\text{MSE}[\hat{f}_{em}] > 2 \sum_{i \neq \text{new}} \mathbb{E}[m_{\text{new}} m_i] + \mathbb{E}[m_{\text{new}}^2]$}{
        $P \leftarrow P - f_{new}$ \tcp*{remove $f_{new}$ from $P$}
        \KwRet $f_{new}$\;
      }
    }
  }

\end{algorithm2e}
Our final algorithm is listed in Algorithm \ref{alg:greedyensemble}. This is a greedy algorithm which adds a new model in each step until Inequality \ref{eq:growingensemble} cannot be satisfied by any remaining model. It should be noted that a greedy algorithm does not guarantee to return the best performing ensemble, which can only be determined by an exhaustive search over all $2^N - 1$ model combinations. 

Our pool of models contains sepsis and non-sepsis patients' models, however, non-sepsis models are often more similar to each other and have a low error, because they usually predict a severity score between 0 and 1 and the number of non-sepsis patients exceeds the number of sepsis patients by more than a factor of 3 (see Table \ref{tab:data_splits}). At the same time, our main objective is sepsis prediction, thus we prioritize the selection of sepsis models in function \texttt{nextModel} by first going through the list of sepsis models and only if no suitable sepsis model is found, we then go through the list of non-sepsis models. We omitted this prioritization strategy in Algorithm \ref{alg:greedyensemble} for reasons of clarity.

We also compare two different weighting schemes for our ensemble. In the uniformly weighted case, the weights $w_i$ are always set to $\frac{1}{N}$:

\begin{equation}
\label{eq:ensemble}
y_T = \sum_{i=1}^N{w_i \cdot y^{(i)}}
\end{equation}
The prediction $y$ at time point $T$ is simply the arithmetic mean of all predictions of the $M$ individual models. As demonstrated in Section \ref{sec:dpensembling}, this weighting scheme delivers the best tradeoff between privacy and accuracy by employing privacy protection with minimal variance.

We additionally employ a more sophisticated method for combining learners that uses a weighting scheme based on the model predictions and previous expert labels. Here, the weights of each individual model is determined by the accuracy at which the model was able to predict the patient's previous labels. This method is connected to the \emph{mixture of experts} strategy \citep{JacobsETAL91}, but instead of using a gating mechanism for selecting the best experts we apply a soft weighting scheme to get the optimal combination. Here, the weights $w_i$ for Equation \ref{eq:ensemble} are calculated using the following expression:
\begin{equation*}
w_i^{(T)} = \frac{1}{C}\sum_{t=1}^{T-1}{ \frac{1}{1 + |y_t^{(i)} - \hat{y}_t|^2} }
\end{equation*}
In words, the weight $w_i^{(T)}$ reflects the accuracy by which model $i$ predicted the label in previous time steps ($1$ to $T-1$). The value $1/C$ is a normalization factor such that $\sum{w_i^{(T)}} = 1$. The idea of this weighting scheme is motivated by our label collecting method: On each day, the senior physicians assign labels to each patient in the intensive care unit (ICU). Thus, a theoretical online-learning algorithm has access to previous expert labels and this information can be exploited to tune weights without retraining the model. 

The code for training, inference, and evaluation of the sepsis prediction model as well as the implementation of the ensemble growing strategy described in Algorithm \ref{alg:greedyensemble} is available on \texttt{github}.\footnote{\url{https://github.com/StatNLP/sepens/}}

\section{Experiments}
\label{sec:experiments}

In this section, we define the patient cohort and how features and labels were obtained for the sepsis prediction task. We compare the fully trained model and the ensemble in terms of prediction accuracy in AUROC, and we empirically show the ensemble's insensitivity to privacy leakage during a membership attack, and evaluate its prediction accuracy with respect to a given privacy budget.

\subsection{Data Cohort}
\label{sec:datacohort}

Our data is based on a PDMS system running at the University Medial Centre in Mannheim, Germany (UMM). The UMM is a 1,352-bed tertiary care center operating a 22-bed interdisciplinary surgical ICU. The hospital is a center of the Acute Respiratory Distress Syndrome (ARDS) Network Germany. 
Timelines of clinical measurements 
were extracted from the Intellispace Critical Care and Anesthesia (ICCA) system by Philips (Eindhoven, Netherlands). Additional demographic patient data as well as ICU admission and discharge times were extracted from a HIS system by SAP (Walldorf, Germany). 

Timelines of 42 features were extracted from the ICCA system, and 1 demographic feature, namely age, was extracted from the HIS system. Other demographic features such as gender did not improve performance of our predictive models. See Table \ref{tab:features} in Appendix \ref{app4} for the list of features we use for training our models. 

At the beginning of an admission many clinical measurements are not available. Such measurements are set to standard default values defined by a clinical expert. To account for varying intervals of clinical measurements during hospital stay, we apply a carry-forward strategy where the most recent value is ``carried forward'' until a new value is available. 
Based on the time lines of varying intervals, we discretize the time lines into uniform steps of 30 minutes during the patient's ICU stay. All values are standardized by calculating $z$-scores, i.e. $z=\frac{x-\mu}{\sigma}$, where $\mu$ is the mean and $\sigma$ is the standard deviation of the population.

\subsubsection{Expert Labels}

Sepsis is a complex concept with a wide range of clinical symptoms. Established definitions aim to operationalize this concept by combining an suspected or existing infection with clinical conditions such as SIRS (Sepsis-1/2) or SOFA (Sepsis-3). 
These definitions are very important in clinical practice, however, they can introduce problems of circularity for machine learning models if the criteria defining a condition are used to predict the very same condition. The problem of circularity in machine learning has been discussed in a broader context in \cite{RiezlerHagmann22}, and for the specific problem of sepsis prediction in \cite{schamoniETAL19}.
We thus established a questionnaire that collects expert opinions on a daily basis \citep{LindnerSchamoniKirschningETAL22}. The questionnaire concerns several aspects in ICU practice and was developed in close cooperation with the senior physicians at the ICU. 
The main goal of the questionnaire is to capture expert opinions that are often based on complex clinical concepts and are thus not fully reflected by established operationalizations. 
On every day, we ask the senior intensivists to assign a current working diagnosis that is not based purely on clinical criteria, but on their experience and their own opinion. The working diagnoses are put on a 5-point scale where 0 stands for ``Neither SIRS nor Sepsis'', and 4 stands for ``Septic Shock''. See Table \ref{tab:explabels} for the complete list of working diagnoses.

\begin{table}[ht]
\centering
\caption{
List of possible expert labels for current working diagnoses.
}
\begin{tabular}{cl}

{\bf Value}&{\bf Working diagnosis}\\
\hline
0 & Neither SIRS nor Sepsis\\
1 & SIRS\\
2 & Sepsis\\
3 & Severe Sepsis\\
4 & Septic Shock\\
\hline

\end{tabular}
\label{tab:explabels}
\end{table}

The questionnaire is filled out on a daily basis at 2~p.m. At the beginning of the survey, the labels were assigned for the preceding 24h window, i.e., from 2~p.m. on the preceding day to 2~p.m. on the current day. Later in the survey, the senior intensivists were asked to set a 6h window of change if a given label differs from the previous label. These new 6h-intervals divide the 24~h window previously in use. To balance the error introduced by the difference of the true sepsis onset to the sepsis labeling time, we set the assumed time of sepsis onset to the center of the interval, e.g., to 2~a.m. for the 24h window, 5~a.m. for the 2~a.m.--8~a.m. window, 11~a.m. for the 8~a.m.--2~p.m. window, 5~p.m. for the 2~p.m.--8~p.m. window, and 11~p.m. for the 8~p.m.--2~a.m. window.

\subsubsection{Data Filtering and Splitting}

Our goal is to learn a model that is able to make timely and accurate predictions on a patient's sepsis outcome. As we are interested in prediction, we excluded on-admission sepsis cases where we defined an on-admission case as a patient who received the first sepsis expert label within the first 48h after ICU admission. We also removed patients that stayed less than 16h in the ICU. Both filtering steps reduced the total number of patients in our cohort from 1,961 to 1,275. To address the problem of having very different distributions in train, validation and test splits, we employed a sampling scheme where we first sort non-sepsis and sepsis patients by length of stay and sepsis onset, respectively, and then sample from groups of four consecutive patients to randomly assign them to four partitions, namely A, B, C, and D. The results of our sampling scheme are listed in Table \ref{tab:data_splits}.

\begin{table}[!ht]
\centering
\caption{
Distribution of hospital stay times and sepsis onset in hours for sepsis and non-sepsis patients across the four partitions.
}
\begin{tabular}{cccc}

& {\bf Patients} & \multicolumn{1}{c}{\bf Non-sepsis pat. stay [h]} & \multicolumn{1}{c}{\bf Sepsis pat. onset [h])}\\
{\bf Partition} & {\bf (Non-/Sepsis)} & {\bf Median/Min/Max} & {\bf Median/Min/Max} \\
\hline
{\bf A} & $319~(245/74)$ & $56.5/ 15.5/ 2253.0$ & $159.0/50.0/1394.0$\\
{\bf B} & $319~(245/74)$ & $55.5/ 15.5/ 1218.5$ & $162.0/49.0/972.5\phantom{0}$ \\
{\bf C} & $319~(245/74)$ & $56.0/ 15.5/ 1618.5$ & $161.0/51.0/1257.0$ \\
{\bf D} & $318~(244/74)$ & $56.0/ 15.5/ 831.5\phantom{0}$  & $161.5/48.5/1269.5$ \\
\hline
{\bf Total} & $1275 (979/296)$ & -- & -- \\

\end{tabular}
\label{tab:data_splits}
\end{table}

In our experiments, we applied a 4-fold cross validation scheme where we used two splits as train data, one as validation data, and the final one as test data. In detail, the train data consists of partitions A+B, B+C, C+D, D+A, and the validation and test sets are partitions C and D, D and A, A and B, and B and C, respectively. Table \ref{tab:data_splits} lists statistics of patients and hospital stay times in the partitions. The preprocessed data splits are available for download.\footnote{\url{https://www.cl.uni-heidelberg.de/statnlpgroup/sepsisexp/}} 

\subsubsection{Ethics}

Ethics approval was obtained from the Medical Ethics Commission II of the Medial Faculty Mannheim, Heidelberg University, Germany (reference number, 2016-800R-MA).

\subsection{Experimental Results}

In this section, we compare prediction accuracy and privacy of ensemble and full models. We show how various privacy budgets affect accuracy loss, and how ensembles successfully prevent privacy leakage in the context of a membership attack.

\subsubsection{Prediction Accuracy of Ensemble and Full Model}

In our first experiment, we compare a model that was trained on all available data to an ensemble grown using Algorithm \ref{alg:greedyensemble}. Table \ref{tab:ensemblesizesplit} lists the resulting ensemble sizes and the ratio of sepsis and non-sepsis patients in our final ensembles per split. While the ensembles are of size 40.5 on average, the numbers range from 20 for split-2 to 65 for split-3, which is 50\% less and 60\% more than the average, respectively. The ratio between non-sepsis and sepsis patient models shows a similar high variance ranging from 0.333 to 0.85 depending on the data split. This illustrates that although we tried to make the data splits as similar as possible, the number of patients is still small and individual patients can have a large influence on the composition of the final model. 

\begin{table}[ht]
\centering
\caption{Comparison of total data and resulting ensemble sizes using the growing strategy described in Algorithm \ref{alg:greedyensemble}. 
}
\label{tab:ensemblesizesplit}
\begin{tabular}{lccccc}

 & {\bf split-0} & {\bf split-1} & {\bf split-2} &  {\bf split-3}  &  {\bf average} \\
\hline
{\bf Total data} \\
\hline
\# train patients & $638$ & $638$ & $637$ & $637$ & -- \\ 
\# dev patients   & $319$ & $318$ & $319$ & $319$ & -- \\ 
\# test patients  & $318$ & $319$ & $319$ & $319$ & -- \\ 
\hline
{\bf Ensemble sizes}\\
\hline
\# total models & $37$ & $20$ & $65$ & $40$ & $40.50$ \\
\# non-sepsis models & $17$ & $5$ & $30$ & $15$ & $16.75$ \\
\# sepsis models & $20$ & $15$ & $35$ & $25$ & $23.75$ \\    
non-sepsis/sepsis ratio & $0.85$ & $0.33$ & $0.85$ & $0.60$ & $0.71$ \\
\end{tabular}
\end{table}

The metric to compare the predictive performance of our models is the area under the receiver operator characteristic curve (AUROC). The AUROC represents the curve of sensitivity-specificity pairs at particular decision thresholds. In the case of time series with prediction intervals, a common practice of calculating the AUROC is to consider sensitivity and specificity at all timesteps of interest, that is the interval itself and the preceding time of hospital stay. We follow the standard procedure of evaluating only the first sepsis episode, as subsequent episodes have very different properties due to interventions such as medication, administration of fluids, etc.
Significance levels were computed using a two-sample $t$-test over the means of the two populations. Means are calculated over four cross-validation runs. 

When comparing the predictive performance in terms of AUROC, the ensembles are remarkably better than the full model. For the uniform ensemble (\textit{ensemble-u}), the difference is significant according to a $t$-test at $p<0.05$ for all but one case, that is predicting sepsis 12h before onset using the uniform model (see Table \ref{tab:resultsensemblepnc}).
We attribute this to the growing strategy of Algorithm \ref{alg:greedyensemble} which guarantees to improve generalization based on theoretical results discussed in Section \ref{sec:weightedaverageensemble}. 
The ensemble with weights adjusted due to the history of prediction performance (\textit{ensemble-w}), we observe larger gains for the shorter prediction times. When moving further away from sepsis onset, the prediction performance of the uniform and the weighted ensemble becomes more similar. This might be influenced by the fact that the expert labels have limited time resolution (24h and 6h) such that the prediction interval is closer to the real onset time than the interval values indicate.

\begin{table}[ht]
\centering
\caption{
AUROC of the ensemble for predicting sepsis at different times before onset. Here, the ensemble is generated using Algorithm \ref{alg:greedyensemble} in Section \ref{sec:algorithm}. The preceding $^{*}$ denotes statistically significant difference ($\alpha = 0.05$) compared to the full model. Values in parenthesis are standard deviations across data splits.}
\label{tab:resultsensemblepnc}

\begin{tabular}{rccccc}

 & {\bf 4h} & {\bf 8h} & {\bf 12h} &  {\bf 12h--8h} & {\bf 24h--12h} \\
\hline
{\bf full model} & \phantom{$^*$}$70.80~(1.84)$ & \phantom{$^*$}$69.48~(1.56)$ & \phantom{$^*$}$67.99~(2.08)$ & \phantom{$^*$}$68.75~(1.85)$ & \phantom{$^*$}$67.13~(2.10)$\\
{\bf ensemble-u} &           $^{*}76.15~(2.66)$ &           $^{*}73.47~(2.62)$ & \phantom{$^*$}$70.51~(3.02)$ &           $^{*}72.12~(2.86)$ &           $^{*}70.77~(2.93)$ \\
{\bf ensemble-w} &           $^{*}78.13~(1.07)$ &           $^{*}74.61~(1.67)$ &           $^{*}70.67~(2.00)$ &           $^{*}72.77~(1.90)$ &           $^{*}70.63~(1.93)$ \\

\end{tabular}
\end{table}

\subsubsection{Privacy and Accuracy: Attacks and Defenses}

Differential privacy has become a privacy standard for privacy-preserving machine learning. However, there exist many forms of differential privacy with different theoretical privacy guarantees. \cite{JayaramanEvans:19} provide a detailed overview on various differentially private machine learning methods in practice. They empirically evaluate two types of privacy attacks, \textit{membership inference} and \textit{attribute inference}. 
We are mostly interested in the former attack, \textit{membership inference}, as it is most relevant for our method of ensembling patient-specific models.

We apply a simple but effective membership inference attack that has been described by \cite{YeomETAL:18}. Here, it is assumed that the attacker has access to the average training loss. While this is not the case in general, an attacker might obtain this single number by a security breach. The attacker could also estimate the average training loss if there exists some knowledge about the original training data distribution. 
To infer membership of a data point, the attacker feeds the data to the model and receives the corresponding prediction. Then, by comparing the prediction to the gold label, the attacker calculates the error (or loss) on this example. Finally, if the calculated loss of the example is smaller than the average training loss, then the example is considered to have been part of the training set. 

In our first experiment on membership attack, we compare the \textit{privacy leakage} \citep{JayaramanEvans:19} of the full model and the uniform ensemble at different privacy budgets $\epsilon$. The privacy leakage, also known as \textit{attacker's advantage} \citep{YeomETAL:18}, is defined as the difference of the true positive rate (TPR) and false positive rate (FPR) of a membership attack. For each evaluated privacy budget, we calculate FPR as the ratio of false positives from the unseen test set, and TPR as the ratio of true positives from the training set. To get additional error estimates, we keep the smaller group fixed and sample sets of an equal size 1,000 times from the larger group, i.e., we fix the test set (negatives) and sample from the training set (positives) 1,000 times to calculate our statistics. 
Figure \ref{fig:privacy_leakage} illustrates that the full model causes privacy leakage at $\epsilon > 10$ while the uniform ensemble is unaffected in terms of membership inference success on all evaluated levels ($10^{-3} \leq \epsilon \leq 10^3$). 

\begin{figure}[ht]
\floatconts
{fig:privacy_leakage} 
{\caption{Privacy leakage for membership inference attacks on the full model and on the uniform ensemble at different $\epsilon$-levels. Privacy leakage is defined as $\text{TPR}-\text{FPR}$, thus it can get values below 0. The full model shows privacy leakage at values $\epsilon > 10$ , while the ensemble preserves privacy at all evaluated $\epsilon$-levels. The vertical bars denote the 2$^{nd}$ and 3$^{rd}$ quartile and the median.}}
{
	\includegraphics[angle=-90,width=.59\textwidth]{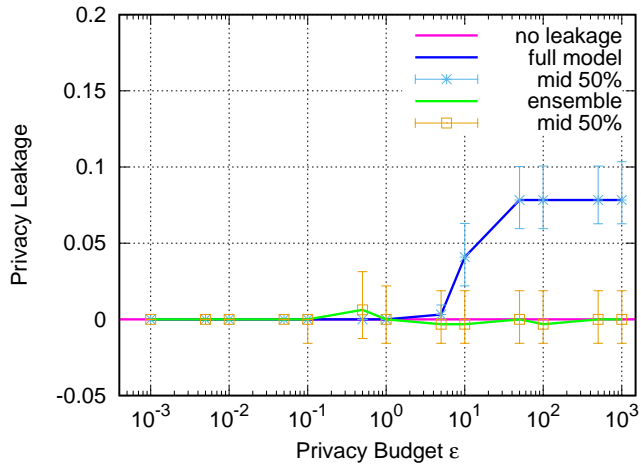}
}
  \end{figure}

In our second experiment, we empirically evaluate the tradeoff between privacy and accuracy which has been discussed in Section \ref{sec:dpensembling}. Privacy is achieved by injecting noise to the model's prediction. This, however, decreases accuracy, so it is important to know how much noise can be injected without sacrificing too much accuracy. We adopt the idea of accuracy loss by \cite{JayaramanEvans:19} and extend it to our AUROC metric. While accuracy typically ranges from 0 to 1, the AUROC ranges from 0.5 (random guessing) to 1.0 (every prediction correct). We thus employ an AUROC-adapted accuracy loss calculation:
\begin{equation*}
    \textit{Accuracy loss}_{AUROC} = 1 - \frac{(2\cdot\textit{AUROC of Private Model})-1}{(2\cdot\textit{AUROC of Non-Private Model})-1}
\end{equation*}
For each $\epsilon$-level and for each prediction interval, we calculate the accuracy loss for both the full model and the uniform ensemble. Throughout all intervals, our experiments indicate that the ensemble's accuracy degrades at privacy budget levels approximately two orders of magnitude smaller than the full model's accuracy. For smaller intervals (Figure \ref{fig:accuracy_loss} $a$,$b$,$c$) the variance between different data splits is high, while for wider prediction intervals (Figure \ref{fig:accuracy_loss} $d$,$e$) the variance is smaller, quantile markings are closer, and curves are smoother. 
\begin{figure}[!ht]
\floatconts
{fig:accuracy_loss}
{\caption{Accuracy loss of the full model and of the uniform ensemble at different $\epsilon$-levels and prediction times before Sepsis onset. Subfigures (\textit{a}), (\textit{b}), and (\textit{c}) show the accuracy loss for the 4h, 8h, and 12h prediction time, subfigures (\textit{d}) and (\textit{e}) for the 12-8h and the 12-24h prediction intervals, respectively. The vertical bars denote the 2$^{nd}$ and 3$^{rd}$ quartile and the median.}}
{%
\subfigure{%
\label{fig:pic1}
\includegraphics[angle=-90,width=.45\textwidth]{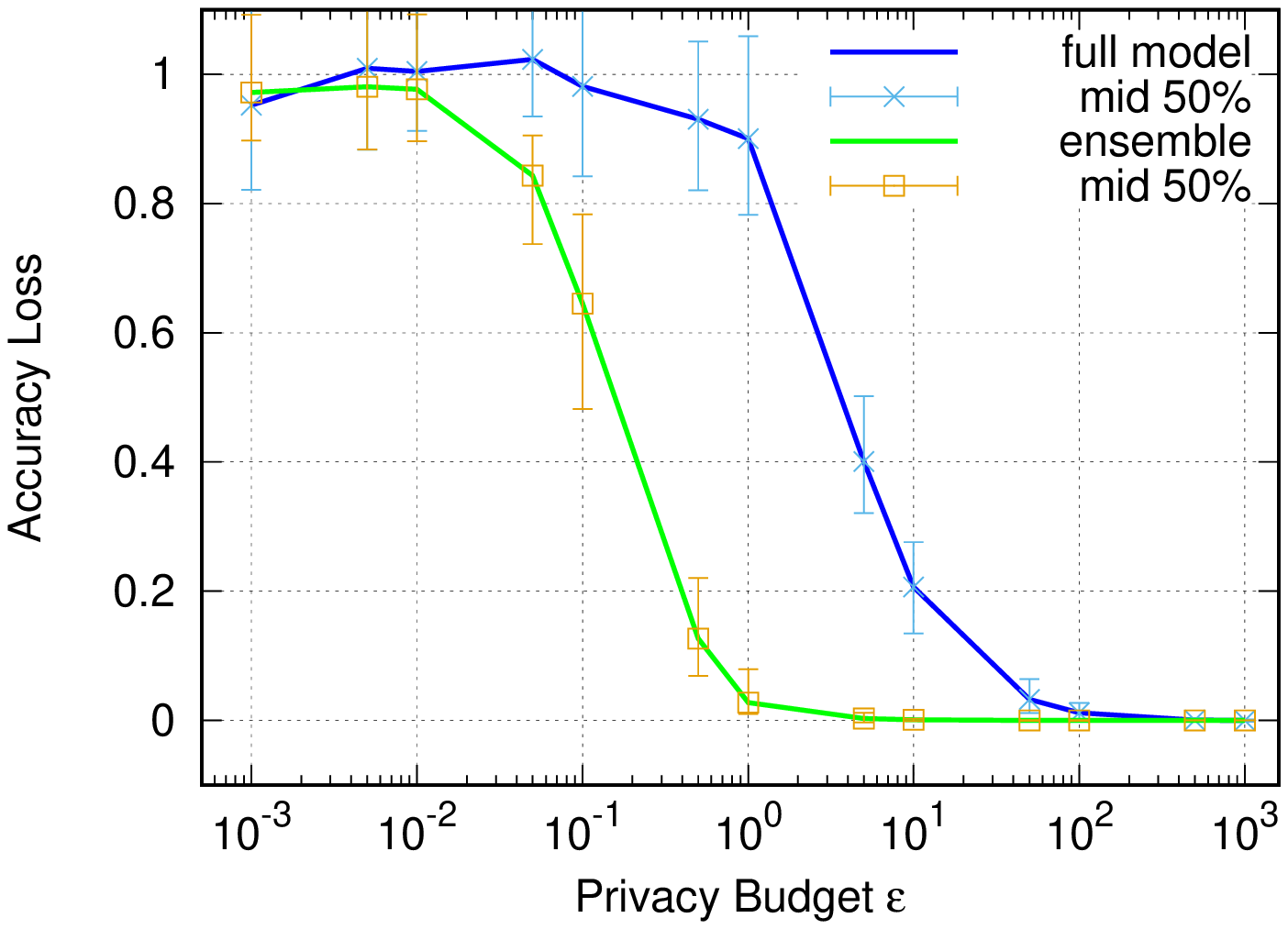}
}\qquad 
\subfigure{%
\label{fig:pic2}
\includegraphics[angle=-90,width=.45\textwidth]{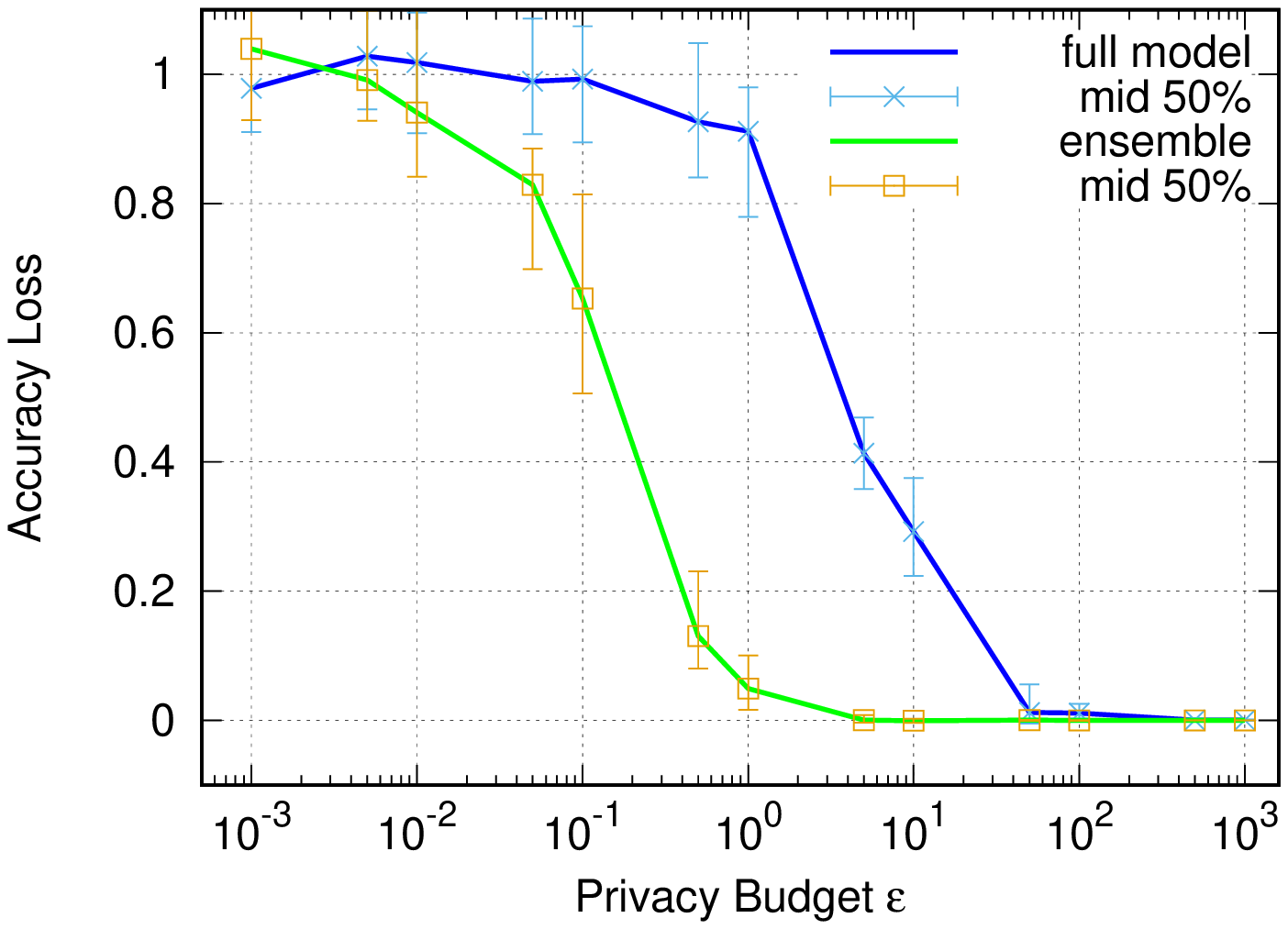}
}
\subfigure{%
\label{fig:pic3}
\includegraphics[angle=-90,width=.45\textwidth]{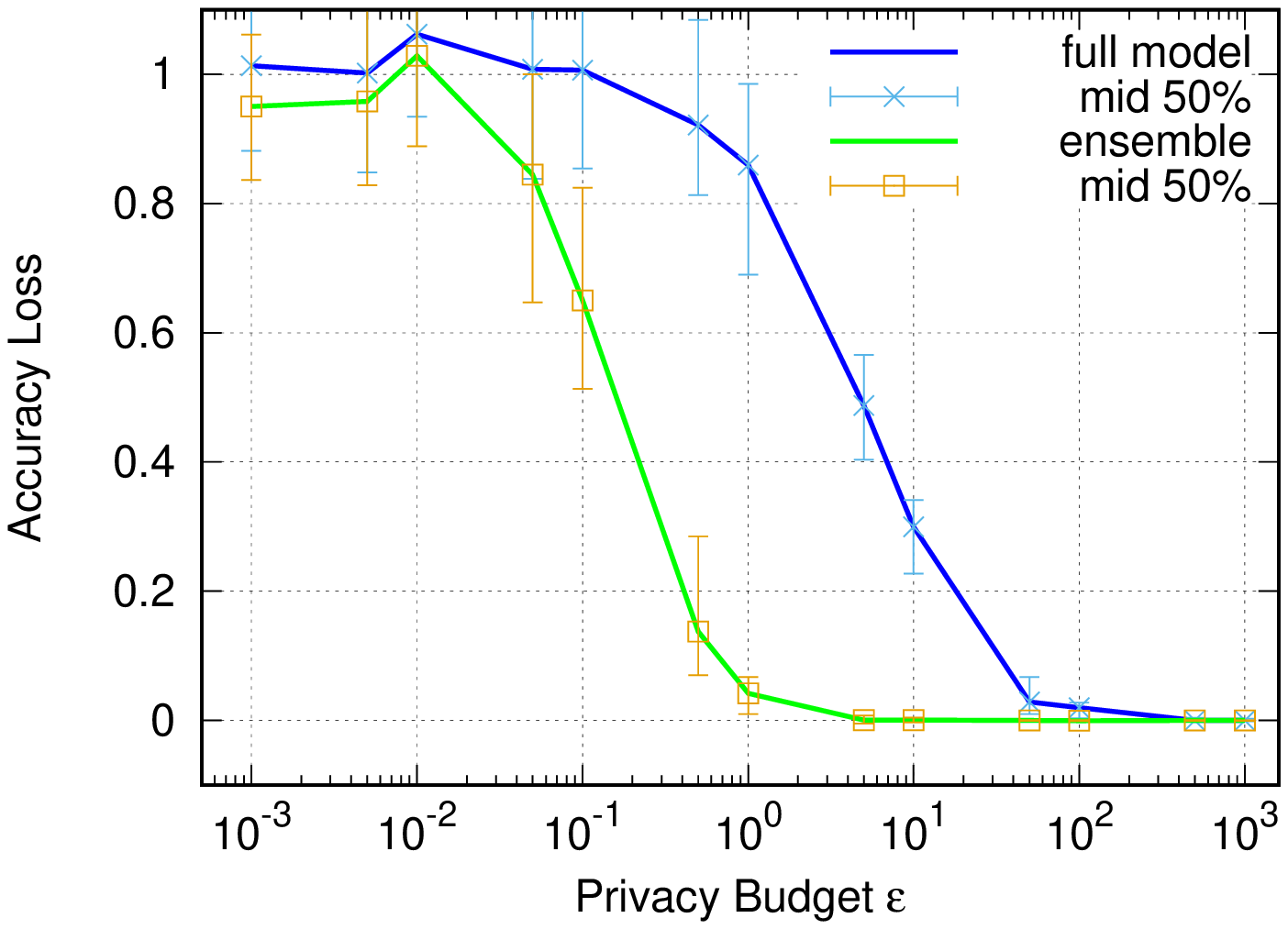}
}\qquad 
\subfigure{%
\label{fig:pic4}
\includegraphics[angle=-90,width=.45\textwidth]{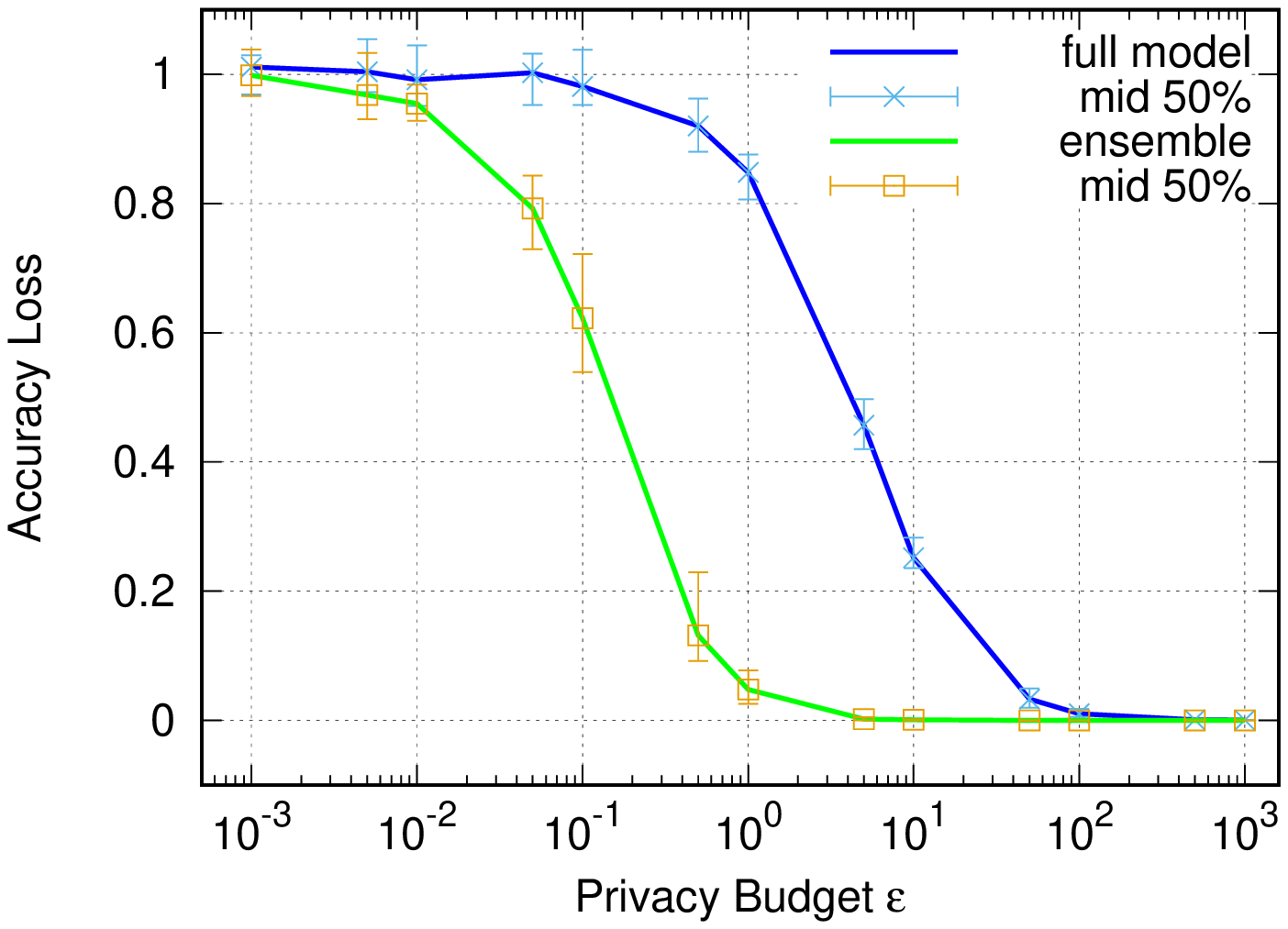}
}
\subfigure{%
\label{fig:pic5}
\includegraphics[angle=-90,width=.45\textwidth]{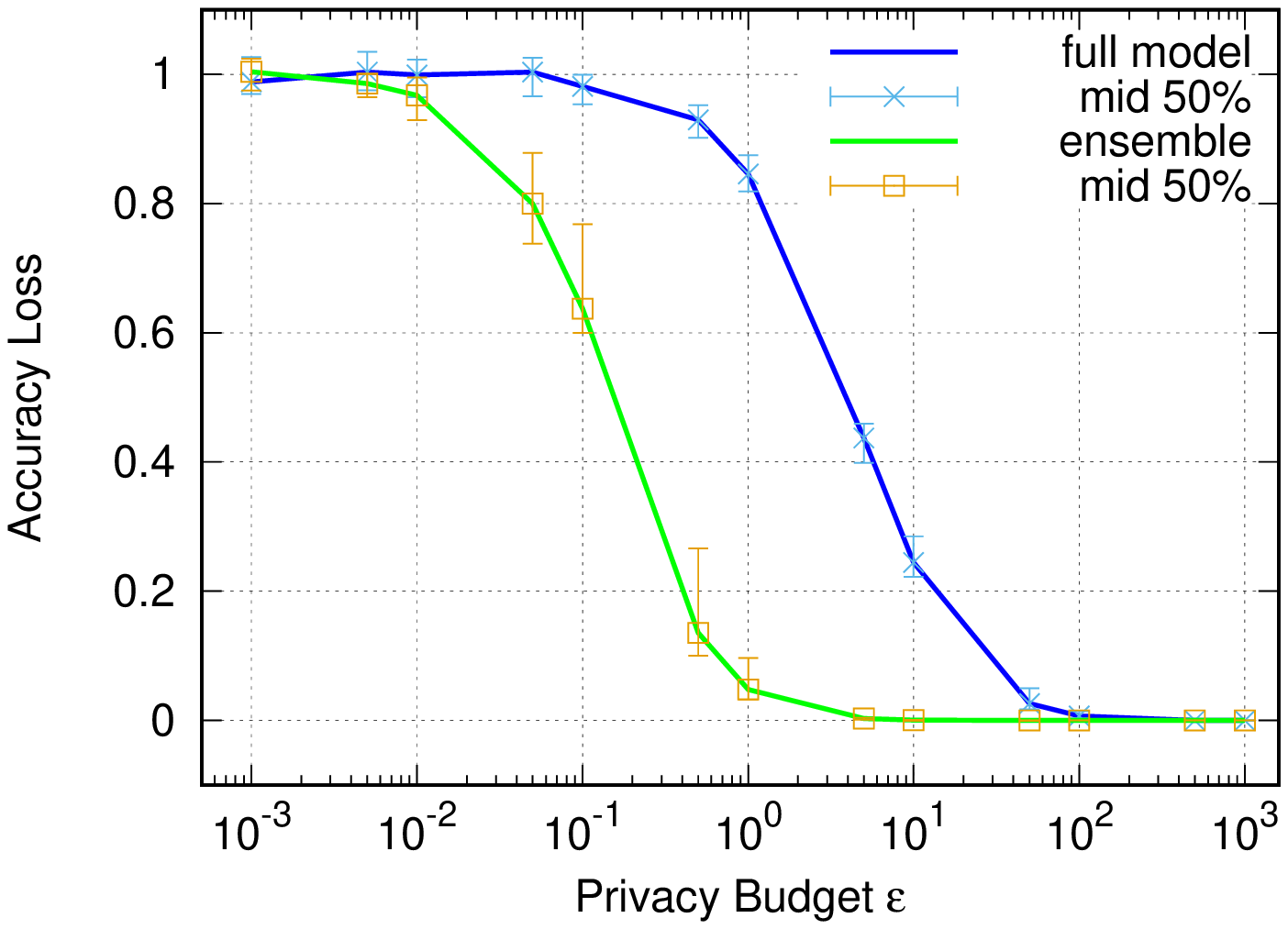}
}
}
\end{figure}

\section{Discussion and Conclusion}

We have shown how to train and combine individual personalized models in an ensemble that has highly desirable properties in the field of medical data analysis. We presented theoretical and empirical results that show that an ensemble of models which were trained on a fraction of the available data can be better than a baseline model that was trained on all data. We applied a strong membership attack and showed that the ensemble successfully prevents privacy leakage while maintaining its accuracy at privacy budgets almost two orders of magnitude smaller than a single fully trained baseline model.

An important type of attack that we did not explicitly evaluate is attribute inference. 
The techniques for attribute inference, however, are indeed based on membership inference. 
In attribute inference, the attacker constructs multiple variants of a candidate patient with and without the attribute in question. Then, similar to membership inference, the prediction error of the model tells the attacker which variant was used in the training dataset. 
Thus, models that successfully prevent membership inference are also able to prevent attribute inference. 

In the field of medical data analysis, ensembling is a promising method for prediction tasks. Explicit sharing of models across different medical communities is desirable and is performed for example in federated learning. However, these techniques are in practice difficult to implement due to various legal constraints and other circumstances such as incompatible interfaces. Protecting privacy in federated learning usually requires very careful injection of noise in the gradients or in the models to not sacrifice performance, which often results in high privacy costs \citep{JayaramanEvans:19}. 
Sharing models implicitly by ensembling and publicly sharing only outputs, where the models remain in a secured, inaccessible area, is a more realistic setup where privacy can be controlled efficiently. 

One limitation of our privacy preserving approach is that we still consume a small privacy budget for each query. Thus, given a fixed privacy budget, subsequently querying our ensemble will eventually use up all the privacy budget available. This problem is addressed by the  Private Aggregation of Teacher Ensembles (PATE) approach \citep{PapernotETAL:17}, where an ensemble of teacher models is used to annotate incomplete public data and train a student model on the annotated data. In doing so, the consumed privacy budget does not increase once the student model is trained. 
In principle, our ensemble strategy could be easily extended to a PATE-like scenario. 

In this work, we follow the standard mechanism for privacy protection by adding randomness at certain locations in the machine learning model. 
Another way of improving privacy could be to exploit the intrinsic randomness of ensembling algorithms such as subsample-and-aggregate \citep{NissimRS07,JordonYS19} or bagging \citep{LiuJG21}. 
An interesting direction for future research is an investigation of data augmentation techniques for privacy protection \citep{YunHCOYC19,LamSR22}.
In these approaches, data are randomly cut and recombined to generate artificial training data that have the potential to protect the privacy of the original data. 

The dataset we use throughout our experiments is quite imbalanced, i.e., the ratio between sepsis and non-sepsis patients is 1:3.3 (see Table \ref{tab:data_splits}). This is mainly due to our filtering where we remove a large amount of short term ICU non-sepsis patients. We found that the ensemble growing strategy we propose does handle imbalances at the ratio mentioned above very well. In reality, however, ICUs often observe much stronger imbalanced data and whether our method degrades at higher ratios or not is an open question. One aspect that implicitly reduces such imbalances is the fact that in general sepsis patients have longer hospital stay times than non-sepsis patients and thus provide more data points to the model.

The ensemble growing strategy described in Algorithm \ref{alg:greedyensemble} provably improves generalization of the ensemble, but as a greedy algorithm it cannot guarantee to return the optimal model combination. Other methods that jointly train an ensemble \citep{BuschjaegerETAL:20} or apply a non-greedy strategy \citep{ZhouETAL:02} might return better ensembles at the cost of increased complexity. Our simple growing strategy, however, means that our ensembles can be easily updated when new patients' data becomes available: a new model needs to be trained on the new data, and Algorithm \ref{alg:greedyensemble} will integrate the model in the ensemble if it performs well on the validation set and if it is reasonably different to the existing models.

Privacy protection is of increasing importance in the growing field of medical data science. Machine learning models highly benefit from increasing amounts of data, which can potentially compromise the patients' rights if techniques are applied without the privacy aspect in mind. There is a lot of active research in model sharing techniques such as federated learning, however, we demonstrate that output sharing such as ensembling is a simple and effective method to provide strong privacy guarantees without sacrificing performance. We don't propose that model sharing and output sharing should be mutual exclusive; at some levels, model sharing might be better applicable, for example in protected in-house scenarios. In other scenarios, where privacy is defined by differing regulations or laws, for example in a national or international context, output sharing might provide an avenue that is simpler to implement and at the same time provides very strong privacy guarantees.

\section{Acknowledgements}

This research has been conducted in project SCIDATOS (Scientific Computing for Improved Detection and Therapy of Sepsis), funded by the Klaus Tschira Foundation, Germany (Grant number 00.0277.2015).

\bibliography{references}

\appendix

\section*{Appendices}
\addcontentsline{toc}{section}{Appendices}
\renewcommand{\thesubsection}{\Alph{subsection}}

\subsection{Model Architecture and Meta-Parameters}
\label{app1}

The details of our general model architecture are as follows. Our model uses LSTM-cells \citep{HochreiterSchmidhuber:97} to model long and short dependencies in the time series data. 
Each patient's stay is divided into 24h windows with 12h of overlap. For example, a 48h stay will be divided into three 24h windows during training. The motivation for using overlapping windows here is that important changes in clinical measurements should not solely occur on one end of a window, but also in the middle of such a sequence so the model has access to more context. 
Our model has 2 hidden LSTM-layers with 200 units each. The input layer takes a 43-dimensional feature vector, the output is the severity score. We train the model with a minibatch size of 20 with mean squared error as the optimization metric. We also apply gradient clipping of 0.25 and set dropout to 0.2 for the hidden layers during training. 

\subsection{Optimality in Weighted Averaging Ensembling}
\label{app2}

\begin{lemma}
Let $\mathfrak{F}=\{\hat{f}_i : i=1,\ldots,N\}$ be a set of regression estimates with covariance matrix $\mathbf{C}=[C_{ij}]_{i,j = 1,\ldots,N}$. Then choosing the weights according to
\begin{align*}
    \mathbf{w}^\ast = \frac{\mathbf{C}^{-1} \mathbf{1}_N}{\mathbf{1}_N^\intercal \mathbf{C}^{-1} \mathbf{1}_N}
\end{align*}
where $\mathbf{1}_N$ is an $N$-vector whose components are all $1$ yields the ensemble estimator with the lowest possible MSE for $\mathfrak{F}$.
\end{lemma}

\begin{note}
If $\mathbf{C}$ is a diagonal matrix then the optimal weights are
\begin{align*}
    w^{\ast}_i = \frac{\frac{1}{\text{MSE}(\hat{f}_i)}}{\sum_{j=1}^N \frac{1}{\text{MSE}(\hat{f}_j)}}  \text{ .}
\end{align*}
If further $\text{MSE}(\hat{f}_1) = \ldots =\text{MSE}(\hat{f}_N)$ the optimal weights are
\begin{align*}
    w^{\ast}_i = \frac{1}{N} \text{ .}
\end{align*}
\end{note}
An important consequence of this Lemma is the fact that if the weights are chosen in an optimal way, then $\text{MSE}(\hat{f}_{em})  \leq \min_{j=1,\ldots,N} \text{MSE}(\hat{f}_{j})$.

\vspace{.5em}
\noindent{}Summary:
 \begin{enumerate}
     \item The MSE of an ensemble estimator of a collection $\mathfrak{F}$ is completely determined by the covariance matrix $\mathbf{C}$. 
     \item In practice $\mathbf{C}$ is unknown and must be replaced by an estimator $\hat{\mathbf{C}}$. In the usual setting where training, test and validation data are drawn from the same distribution the covariances $C_{ij}$ can be estimated by validation set sample means.
     \item Regarding the optimal weights, the invertibility of $\mathbf{C}$ is of direct importance. Given that in practice we need to base our calculation on $\hat{\mathbf{C}}$ which should aim at choosing $\mathfrak{F}$ in such a way that the inversion of $\mathbf{C}$ (and in consequence  $\hat{\mathbf{C}}$) is numerically stable and well conditioned. One way to achieve this is to choose $\mathfrak{F}$ in such a way that $\mathbf{C}$ is a diagonal matrix.
 \end{enumerate}

\subsection{Differential Privacy} 
\label{app3}

Differential privacy \citep{Dwork:06,DworkRoth:14} is based on guarantees that a randomized algorithm behaves similarly on similar input databases, i.e., on databases differing in one data point. Let $D, D'\in \mathcal{D}$ be two data sets that are obtained from one another by removing one data point, called neighboring data sets, and denoted by $D \sim D'$. Furthermore, let $\mathcal{A}$ be a randomized algorithm producing an output in the space $\mathcal{O}$  on input data in $\mathcal{D}$.

\begin{definition}
A randomized algorithm $\mathcal{A}$ is $(\epsilon,\delta)$-differentially private if for all data sets $D, D'\in \mathcal{D}$, and all subsets of outcomes $S \subseteq \mathcal{O}$,
\[
Pr[\mathcal{A}(D) \in S] \leq e^{\epsilon} Pr[\mathcal{A}(D') \in S] + \delta,
\]
where $\epsilon$ is the privacy budget, and $\delta$ is the failure probability.
\end{definition} 

The concept of differential privacy implies that the level of protection of data is always lowered when a model trained on this data is queried. The  privacy budget $\epsilon$ is thus reduced by each model query. A common way to lower this reduction is to add noise to the model's answer of the query. 

Let us consider deterministic functions $f: \mathcal{D} \rightarrow \mathbb{R}^d$ as fundamental types of database queries. The amount of noise that is required to preserve privacy of a function $f$ is proportional to its sensitivity, i.e., to the maximum change in the output of $f$ over all possible inputs:
\begin{definition}
The $\ell_1$-sensitivity of a function $f: \mathcal{D} \rightarrow \mathbb{R}^d$ is 
\[\Delta f = \max_{D \sim D'} \norm{f(D) - f(D')}_1.\]
\end{definition}

A standard technique to achieve differential privacy is the Laplace mechanism that perturbs each coordinate of a function with noise drawn from a Laplace distribution, with variance proportional to the sensitivity of the function (divided by $\epsilon$):
\begin{definition}
Given a function $f: \mathcal{D} \rightarrow \mathbb{R}^d$, the Laplace mechanism is defined as $f(D) + (W_1, \ldots, W_d)$ where $W_i$ are i.i.d. random variables drawn from $\mathrm{Lap}(\Delta{f}/\epsilon)$, and $\mathrm{Lap}(s)$ is a Laplace distribution with mean $0$ and variance $2s^2$.
\end{definition}

\subsection{Features for Sepsis Prediction Task}
\label{app4}
\begin{table}[ht]
\centering
\caption{
List of the 43 features we used for training our models for the Sepsis prediction task. Features can be readings from vital monitors (e.g., heart rate, blood pressure), lab results (e.g., bilirubin, creatinine), or static demographic features (age). 
}
\label{tab:features}
\begin{tabular}{llll}

\hline
 Age        & Arterial pH    & Urine output       & Procalcitonin (PCT)    \\
 Heart rate & Leukocytes     & Blood glucose      & $\Delta$ Temperature    \\
 Lactate    & Bicarbonate    & Stroke volume      & Alanine transaminase   \\
 Creatinine & Base excess    & Horowitz index     & BUN/Creatinine ratio   \\
 Bilirubin  & Lymphocytes    & Partial CO$_2$     & Aspartate transaminase    \\
 Sodium     & Net balance    & Respiratory rate   & Oxygenation saturation \\
 Potassium  & Quick score    & Calcium (ionized)  & C-reactive protein (CRP)\\
 Hemoglobin & Systolic BP    & Heart time volume  & Respiratory minute volume\\
 Chloride   & Temperature    & Oxygen saturation  & Fraction of inspired O$_2$\\
 SVRI       & Diastolic BP   & Pancreatic lipase  & Partial pressure art. O$_2$\\
 Mean BP    & Thrombocytes   & Blood urea nitrogen& \\
\hline

\end{tabular}

\end{table}

\end{document}